\documentclass{article}

\usepackage[preprint, nonatbib]{neurips_2025}

\usepackage[utf8]{inputenc} \usepackage[T1]{fontenc}    \usepackage[hidelinks]{hyperref}       \usepackage{url}            \usepackage{amsfonts}       \usepackage{nicefrac}       \usepackage{microtype}      

\usepackage[textsize=tiny]{todonotes}

\usepackage[
    backend=biber,
    maxbibnames=99, 
    maxcitenames=2, 
    natbib=true, 
    style=numeric, 
    date=year,
    isbn=false, 
    url=false, 
    eprint=false, 
    giveninits=true, 
    sortcites
]{biblatex} 

\DeclareFieldFormat{note}{\mkbibparens{License: #1}}

\AtEveryBibitem{%
  \ifentrytype{online}{%
    \clearfield{url}%
    \clearfield{eprint}%
    \clearfield{urlyear}    \clearfield{urlmonth}%
  }{}%
}

\AtEveryBibitem{\clearname{editor}}

\usepackage[german, main=english]{babel}
\usepackage[autostyle]{csquotes}

\usepackage{amsmath}
\usepackage{amssymb} \usepackage{mathtools} \usepackage{bm} \usepackage{dsfont} \usepackage{amsthm} 
\usepackage[acronym]{glossaries-extra}   \GlsXtrEnableEntryUnitCounting      {abbreviation}    {1}    {document}  
\glsdisablehyper

\usepackage[capitalize,noabbrev,english]{cleveref}

\usepackage{xcolor}
\usepackage{enumitem}   \usepackage{xspace}     
\usepackage{siunitx}

\graphicspath{ {./figures/} }
\usepackage[
     format=plain,           labelformat=simple,      labelsep=period,      textformat=simple,           labelfont=footnotesize,sl,      textfont=footnotesize,      ]{caption}
\usepackage{graphicx}
\usepackage{subcaption}
\usepackage{tabularx}   \usepackage{multirow}   \usepackage{makecell}   \usepackage{booktabs}   

\usepackage{pgfplots}
\pgfplotsset{width=10cm,compat=1.8}
\usepgfplotslibrary{groupplots, statistics, fillbetween}
\usetikzlibrary{calc, spy, decorations.markings}

\usepgfplotslibrary{external}
\makeatletter
\renewcommand{\todo}[2][]{\tikzexternaldisable\@todo[#1]{#2}\tikzexternalenable}
\makeatother

\newcommand{\bszero}{\boldsymbol{0}}
\newcommand{\bsA}{\boldsymbol{A}}

\newcommand{\bsG}{\boldsymbol{G}}

\newcommand{\bsI}{\boldsymbol{I}}
\newcommand{\bsJ}{\boldsymbol{J}}

\newcommand{\bsL}{\boldsymbol{L}}
\newcommand{\bsM}{\boldsymbol{M}}

\newcommand{\bsR}{\boldsymbol{R}}

\newcommand{\bsT}{\boldsymbol{T}}
\newcommand{\bsU}{\boldsymbol{U}}
\newcommand{\bsV}{\boldsymbol{V}}
\newcommand{\bsW}{\boldsymbol{W}}
\newcommand{\bsX}{\boldsymbol{X}}

\newcommand{\bsa}{\boldsymbol{a}}
\newcommand{\bsb}{\boldsymbol{b}}

\newcommand{\bsf}{\boldsymbol{f}}

\newcommand{\bsq}{\boldsymbol{q}}

\newcommand{\bsu}{\boldsymbol{u}}

\newcommand{\bsx}{\boldsymbol{x}}

\newcommand{\bsz}{\boldsymbol{z}}

\newcommand{\bsomega}{\boldsymbol{\omega}}

\newcommand{\bsSigma}{\boldsymbol{\Sigma}}

\newcommand{\cH}{{\mathcal{H}}}

\newcommand{\bbR}{{\mathbb{R}}}

\newcommand{\partdiff}[3][]{\frac{\partial^{#1}#2}{\partial#3^{#1}}}
\newcommand{\hessian}[2]{\frac{\partial^2{#1}}{\partial{#2}\partial{#2}}}

\newcommand\norm[2][]{\left\lVert#2\right\rVert_{#1}}
\DeclareMathOperator{\diag}{diag}

\DeclareMathOperator*{\argmin}{arg\,min}

\newcommand{\hamiltonian}{\cH}

\theoremstyle{plain}
\newtheorem{theorem}{Theorem}[section]

\newtheorem{lemma}[theorem]{Lemma}

\theoremstyle{definition}

\theoremstyle{remark}

\newcommand{\revision}[1]{#1}

\newcommand{\mymodel}{\gls{sPHNN}\xspace}
\newcommand{\mymodels}{\glspl{sPHNN}\xspace}

\newcommand{\NODE}{\gls{NODE}\xspace}
\newcommand{\NODEs}{\glspl{NODE}\xspace}

\newcommand{\CapNODEs}{\Glspl{NODE}\xspace}

\newcommand{\PHNN}{PHNN\xspace}
\newcommand{\PHNNs}{PHNNs\xspace}

\newcommand{\structurematrix}{structure matrix\xspace}

\newcommand{\inputTikzWithExternalization}[2]{%
    \includegraphics{externalized_figures/#1.pdf}%
}

\newcommand{\numsvdmodes}{\num{40}\xspace}

\newcommand{\giturl}{\url{https://github.com/CPShub/sphnn-publication}}

\newabbreviation[type=acronym]{0GAS}{0-GAS}{globally asymptotically stable at zero}
\newabbreviation[type=acronym]{AE}{AE}{autoencoder}
\newabbreviation[type=acronym]{APRBS}{APRBS}{amplitude-modulated pseudo-random binary sequence}
\newabbreviation[type=acronym]{bPHNN}{bPHNN}{bounded PHNN}
\newabbreviation[type=acronym]{DOF}{DOF}{degree of freedom}
\newabbreviation[type=acronym]{FFNN}{FFNN}{feedforward neural network}
\newabbreviation[type=acronym]{GENERIC}{GENERIC}{general equation for the nonequilibrium reversible-irreversible coupling}
\newabbreviation[type=acronym]{HNN}{HNN}{Hamiltonian neural network}
\newabbreviation[type=acronym]{ICNN}{ICNN}{input convex neural network}
\newabbreviation[type=acronym]{FICNN}{FICNN}{fully input convex neural network}
\newabbreviation[type=acronym]{FE}{FE}{finite element}
\newabbreviation[type=acronym]{FOM}{FOM}{full order model}
\newabbreviation[type=acronym]{IC}{IC}{initial condition}
\newabbreviation[type=acronym]{INN}{INN}{invertible neural network}
\newabbreviation[type=acronym]{IPHS}{IPHS}{irreversible port-Hamiltonian system}
\newabbreviation[type=acronym]{ISPHS}{ISPHS}{input-state port-Hamiltonian system}
\newabbreviation[type=acronym]{ISOPHS}{ISOPHS}{input-state-output port-Hamiltonian system}
\newabbreviation[type=acronym]{ISS}{ISS}{input-to-state stability}
\newabbreviation[type=acronym]{LNN}{LNN}{Lagrangian neural network}
\newabbreviation[type=acronym]{LTI}{LTI}{linear time invariant}
\newabbreviation[type=acronym]{ML}{ML}{machine learning}
\newabbreviation[type=acronym]{MSE}{MSE}{mean squared error}
\newabbreviation[type=acronym]{NN}{NN}{neural network}
\newabbreviation[type=acronym]{NODE}{NODE}{neural ordinary differential equation}
\newabbreviation[type=acronym]{NSDE}{neural SDE}{neural stochastic differential equation}
\newabbreviation[type=acronym]{ODE}{ODE}{ordinary differential equation}
\newabbreviation[type=acronym]{PCC}{PCC}{Pearson correlation coefficient}
\newabbreviation[type=acronym]{PDE}{PDE}{partial differential equation}
\newabbreviation[type=acronym, shortplural=PHS]{PHS}{PHS}{port-Hamiltonian system}
\newabbreviation[type=acronym]{PNN}{PNN}{Poisson neural network}
\newabbreviation[type=acronym]{POD}{POD}{proper orthogonal decomposition}
\newabbreviation[type=acronym]{ReLU}{ReLU}{rectified linear unit}
\newabbreviation[type=acronym]{RMSE}{RMSE}{root mean squared error}
\newabbreviation[type=acronym]{ROA}{ROA}{region of attraction}
\newabbreviation[type=acronym]{ROM}{ROM}{reduced order model}
\newabbreviation[type=acronym]{SDE}{SDE}{stochastic differential equation}
\newabbreviation[type=acronym]{sNODE}{sNODE}{stable NODE}
\newabbreviation[type=acronym]{SVD}{SVD}{singular value decomposition}
\newabbreviation[type=acronym]{sISPHS}{sISPHS}{stable ISPHS}
\newabbreviation[type=acronym]{sPHNN}{sPHNN}{stable port-Hamiltonian neural network}

\definecolor{CPSgreen}{RGB}{22,164,138}
\definecolor{CPSlightblue}{RGB}{104,143,198}
\definecolor{CPSdarkblue}{RGB}{67,83,132}
\definecolor{CPSgrey}{RGB}{204, 204, 204}
\definecolor{CPSorange}{RGB}{246,163,21}
\definecolor{CPSred}{RGB}{194,76,76}

\pgfplotsset{
    myboxplot/.style={
        boxplot={
        every box/.style={
            thick,
            },
                                every median/.style={
            very thick,
            CPSlightblue,
            }
        },
        mark=+,
        draw=black,
        fill=white,
    }
}

\pgfplotsset{
    myplotstyle/.style={
        label style={font=\scriptsize, inner sep=0},
        tick label style={font=\scriptsize},
        title style={font=\footnotesize, yshift=-1.5ex},
        legend style={
            nodes={scale=0.6, transform shape},
        },
                        label style={inner sep=0pt},         tick label style={inner sep=2pt},     }
}

\pgfplotsset{
    mycolorbarstyle/.style={
        label style={font=\scriptsize, inner sep=0},
        tick label style={font=\scriptsize},
        title style={font=\footnotesize, yshift=-1.5ex},
        legend style={
            nodes={scale=0.6, transform shape},
        },
                        label style={inner sep=0pt},         tick label style={inner sep=2pt},     }
}

\pgfplotsset{
    mylinestyle/.style={
        mark=none,
        thick,
                draw=black,
        opacity=1.,    },
    myshadestyle/.style={
                        opacity=0.2,
    },
    sPHNNstyle/.style           ={mylinestyle, CPSred, very thick},
    sPHNNshadestyle/.style      ={myshadestyle, CPSred},
    sPHNNLMstyle/.style         ={mylinestyle, densely dashdotted, CPSorange},
    sPHNNLMshadestyle/.style    ={myshadestyle, CPSorange},
    cPHNNstyle/.style           ={mylinestyle, densely dashdotdotted, CPSdarkblue},
    cPHNNshadestyle/.style      ={myshadestyle, CPSdarkblue},
    PHNNstyle/.style            ={mylinestyle, densely dotted, CPSlightblue},
    PHNNshadestyle/.style       ={myshadestyle, CPSlightblue},
    sNODEstyle/.style           ={mylinestyle, semithick, CPSgrey},
    sNODEshadestyle/.style      ={myshadestyle, CPSgrey},
    NODEstyle/.style            ={mylinestyle, CPSgreen, semithick},     NODEshadestyle/.style       ={myshadestyle, CPSgreen},
        excitationstyle/.style      ={mylinestyle, CPSgrey!50!black, dashed},
    groundtuthstyle/.style      ={mylinestyle, black, dashed, ultra thick}, }

\title{Stable Port-Hamiltonian Neural Networks}

\author{%
                                          Fabian J.~Roth, Dominik K.~Klein, Maximilian Kannapinn, Jan Peters, Oliver Weeger\\
  Technical University of Darmstadt, Darmstadt, Germany\\ 
  \texttt{\{\href{mailto:roth@cps.tu-darmstadt.de}{roth},\href{mailto:klein@cps.tu-darmstadt.de}{klein},\href{mailto:kannapinn@cps.tu-darmstadt.de}{kannapinn},\href{mailto:weeger@cps.tu-darmstadt.de}{weeger}\}@cps.tu-darmstadt.de}, \\ \texttt{\href{mailto:peters@ias.tu-darmstadt.de}{peters@ias.tu-darmstadt.de}}
}

\begin{document}

\maketitle

\begin{abstract}
    In recent years, nonlinear dynamic system identification using artificial neural networks has garnered attention due to its broad potential applications across science and engineering. 
However, purely data-driven approaches often struggle with extrapolation and may yield physically implausible forecasts. 
Furthermore, the learned dynamics can exhibit instabilities, making it difficult to apply such models safely and robustly.
This article introduces \glsxtrlongpl{sPHNN}, a machine learning architecture that incorporates physical biases of energy conservation and dissipation while ensuring global Lyapunov stability of the learned dynamics. 
Through illustrative and real-world examples, we demonstrate that these strong inductive biases facilitate robust learning of stable dynamics from sparse data, while avoiding instability and surpassing purely data-driven approaches in accuracy and physically meaningful generalization.
Furthermore, the model's applicability and potential for data-driven surrogate modeling are showcased on multiphysics simulation data.

\end{abstract}

\section{Introduction}\label{sec:introduction}

Recent years have seen a strongly growing interest in using \gls{ML} to identify the dynamics of physical systems directly from observations \cite{legaard2023,lee2021,lai2021,boettcher2023,park2024,wang2021,ramasinghe2023}. 
This approach has diverse applications, ranging from modeling phenomena with partially or entirely unknown underlying physics to creating surrogate and reduced-order models for rapid simulations and digital twins \cite{willard2022, wang2021, kannapinn2024}. In the latter application, data-driven methods can address the drawbacks of traditional numerical methods, which often require tremendous computational resources and expertise~\cite{wang2021,abbasi2024}. However, in turn, classical \gls{ML} algorithms typically require large amounts of training data and lack robustness, as they often suffer from poor extrapolation and are prone to making physically implausible forecasts \cite{abbasi2024,wang2021,willard2022,liu2019a}.
To overcome these limitations, a general consensus is to avoid discarding centuries' worth of established knowledge in physics, which is an implicit consequence of relying solely on data-driven methods. Instead, there is a growing emphasis on incorporating physical priors and inductive biases into \gls{ML} approaches \cite{willard2022,liu2019a}. 
This \emph{physics-guided} \gls{ML} promises improved accuracy, robustness, and interpretability while reducing data requirements \cite{wang2021,liu2019a}. 

Several physics-guided architectures have been proposed for learning dynamical systems. These approaches often parameterize an \gls{ODE} to identify continuous-time dynamics. 
To ensure adherence to the laws of thermodynamics, many architectures build upon frameworks such as Hamiltonian mechanics \cite{greydanus2019,jin2020,choudhary2021}, Lagrangian mechanics \cite{lutter2019,cranmer2020}, Poisson systems \cite{jin2023}, \glsxtrshort{GENERIC} \cite{hernandez2021,zhang2022}, \glspl{PHS} \cite{zhong2020,desai2021,eidnes2023,neary2023,nakano2022}, or the generalized Onsager principle \cite{yu2021}. 
The latter three allow for the modeling of dissipative dynamics, while the former are generally limited to energy-conserving systems. 
Additionally, \glspl{PHS} explicitly incorporate external excitations, making them particularly suited for modeling controlled systems.

The mentioned physical modeling frameworks have implications for the stability of the learned dynamics.
For example, in Hamiltonian neural networks, stable equilibria correspond to extrema in the Hamiltonian energy function. 
However, these effects are often not discussed in the referenced works.
Nevertheless, stability is an important and fundamental property of any dynamic system. 
Guarantees on the stability of the dynamics may improve the model's robustness, help trust their predictions, and can themselves be viewed as physically motivated biases~\cite{erichson2019}.

Given the importance of the stability property, multiple other works have proposed methods to enhance the stability of learned dynamics, often using \NODEs \cite{chen2018} as a basis.
\citet{massaroli2020} propose a provably stable \NODE variant that guarantees the existence of locally asymptotically stable equilibria. 
Multiple studies have combined Lyapunov's theory with \glspl{NN} to guarantee the stability of discrete-time dynamics \cite{lawrence2021,erichson2019} or to estimate the region of attraction for given dynamical systems \cite{richards2018,barreau2024}. 
\textcite{kolter2019} describe another Lyapunov-stable architecture based on \NODEs. 
They concurrently learn unconstrained dynamics with a \NODE and a convex Lyapunov function. Global stability around an equilibrium point is then ensured by projecting the \NODE-dynamics onto the space of stable dynamics as given by the Lyapunov function. 
\citet{takeishi2021} extend this approach from a stable equilibrium to stable invariant sets. Similar projection-based techniques have also been proposed to learn input-output stable dynamics \cite{kojima2022} and dissipative systems \cite{okamoto2024}. 
However, these projection methods can introduce discontinuities, hindering training strategies that rely on integrating trajectories rather than matching state-derivative pairs \cite{schlaginhaufen2021}.
Other approaches employ physical modeling frameworks to formulate the learning problem directly in the subspace of stable dynamics.
\citet{rettberg2024} propose identifying a linear latent \gls{PHS}, which makes enforcing global stability relatively straightforward. 
\citet{yu2021} propose a model based on the generalized Onsager principle that can guarantee bounded solutions and the existence of locally asymptotically stable equilibria. While their evolution equation is similar to \gls{PHS}, their approach does not consider arbitrary, time-dependent external input signals.

This work exploits the connections between stability and physical frameworks and proposes \mymodel for learning \emph{nonlinear} and \emph{globally stable} dynamical systems. 
The method is projection-free and can incorporate input signals. 
\revision{
    We show that the approach is capable of learning equilibrium points, though our numerical evaluations focus on practical applications where the equilibrium is known a priori. In such cases, this prior knowledge can be seamlessly integrated into the proposed architecture, improving prediction accuracy while simultaneously reducing data requirements.
}
\section{Background}\label{sec:background}

\subsection{Stability of dynamical systems}\label{sec:stability}

We consider autonomous dynamical systems as first-order \glspl{ODE}, where the state $\bsx(t)\in\bbR^n$ evolves according to a vector-valued function $\bsf:\bbR^n\rightarrow\bbR^n$:
\begin{equation}\label{eq:autonomous_system}
    \dot\bsx(t)=\bsf(\bsx(t)).
\end{equation}
Suppose $\bsf$ is sufficiently smooth to ensure a unique solution $\bsx(t)$ for any initial condition $\bsx(t_0)$. A solution $\hat\bsx(t)$ is \emph{stable} in the sense of Lyapunov if, for any small perturbation in the initial condition, the perturbed solution $\Bar\bsx(t)$ remains close for all times, that is, if for any $\epsilon>0$, there exists a $\delta>0$ such that $\norm{\hat\bsx(t_0) - \Bar{\bsx}(t_0)} < \delta$ implies  $\norm{\hat\bsx(t) - \Bar\bsx(t)} < \epsilon$ for all $t>t_0$ \cite{malkin1959, verhulst1990}.
A stronger version of stability, namely \emph{asymptotic stability}, is obtained if additionally the perturbed solution $\bar\bsx(t)$ converges to $\hat\bsx(t)$ for $t\rightarrow\infty$ for initial conditions $\norm{\hat\bsx(t_0) - \Bar{\bsx}(t_0)} < \Delta$ with some $\Delta>0$ \cite{verhulst1990}.

Lyapunov stability is a \emph{local} property, in the sense that for sufficiently large perturbations, the perturbed solution may stray arbitrarily far even from asymptotically stable solutions \cite{seydel2010}. 
Quantifying the permissible extent of perturbations on the initial condition such that the perturbed solution still converges to the unperturbed solution leads to the concept of a \emph{basin of attraction} \cite{layek2015}.
Asymptotically stable equilibria, for which the basin of attraction extends to the entire phase space, are called \emph{globally asymptotically stable} \cite{layek2015}. A system with a globally stable equilibrium cannot have any other equilibria, as their existence would imply that there are solutions (namely, the other equilibria) that do not converge to the globally stable equilibrium.

In the following, we consider the case of an equilibrium solution $\bsx(t)=\bszero$, i.e., $\bsf(\bszero)=\bszero$. Placing the equilibrium at the origin is done for ease of notation and is without loss of generality, as arbitrary equilibria can be translated to the origin with a time-independent transformation. 
The stability of the equilibrium can be \revision{shown} through the use of Lyapunov's stability theory by finding a suitable scalar Lyapunov function $V$ \cite{verhulst1990, khalil2002}. Let $V: D\rightarrow\bbR$ be continuously differentiable and positive definite, that is, $V(\bszero)=0$ and $ V(\bsx) > 0$ in $D\setminus\{\bszero\}$, where $D\subseteq\bbR^n$ is a neighborhood of the origin. If the value of $V$ is nonincreasing along trajectories of the system in $D$, then the equilibrium is stable in the sense of Lyapunov. Mathematically, this is the case if $\dot V(\bsx)$ is negative semi-definite:
\begin{equation}\label{eq:decrease_condition}
    \dot V(\bsx)=\partdiff{V}{\bsx}^{\intercal}\dot\bsx = \partdiff{V}{\bsx}^{\intercal}\bsf(\bsx) \leq0\quad\forall\bsx\in D\setminus\{\bszero\}.
\end{equation}
In the case of a negative definiteness ($\dot V(\bsx)<0$ in $D\setminus\{\bszero\}$), \emph{asymptotic} stability can be concluded.
To obtain \emph{global} asymptotic stability, these conditions must be satisfied \emph{globally}, that is, for all ${\bsx\in D=\bbR^n}$, \emph{and} the Lyapunov function must be radially unbounded, i.e., $V(\bsx)\rightarrow\infty$ for $\norm{\bsx}\rightarrow\infty$ \cite{khalil2002}.
The latter condition avoids the existence of solutions that do not converge to the equilibrium by diverging to infinity without violating the decrease condition in \cref{eq:decrease_condition}.

\subsection{Port-Hamiltonian systems}\label{sec:phs}

\Glspl{PHS} theory provides a framework for modeling physical systems that generalizes the underlying mathematical structure of Hamiltonian dynamics. It incorporates the modeling of energy-dissipating elements, which are typically absent in classical Hamiltonian systems \cite{vanderschaft2014}. Additionally, \glspl{PHS} theory has roots in control theory, emphasizing the interactions of dynamic systems with their environment.
In this work, we consider \glspl{PHS} of the form
\begin{equation}\label{eq:isphs_evolution}
    \dot{\bsx} = \left[\bsJ(\bsx) - \bsR(\bsx)\right]\partdiff{\hamiltonian}{\bsx}(\bsx) + \bsG(\bsx)\bsu(t).
\end{equation}
Here, $\bsx(t)\in\bbR^n$ describes the system's state, $\bsu(t)\in\bbR^m$ its input, and $\hamiltonian:\bbR^n\rightarrow\bbR$ denotes the Hamiltonian, which typically represents the total energy of the system \cite{vanderschaft2014, beattie2018}.
The skew-symmetric \emph{structure matrix} ${\bsJ:\bbR^n\rightarrow\bbR^{n\times n}},\,{\bsJ(\bsx)=-\bsJ^{\intercal}(\bsx)}$ describes the conservative energy flux within the system. 
The symmetric positive semi-definite matrix ${\bsR:\bbR^n\rightarrow\bbR^{n\times n}},\,{\bsR(\bsx)=\bsR^{\intercal}(\bsx)\succeq0}$ is called \emph{dissipation matrix} and describes the energy losses. Finally, the \emph{input matrix} ${\bsG:\bbR^n\rightarrow\bbR^{n\times m}}$ describes how energy enters and exits the system via the inputs.
\Glspl{PHS} extend the energy conservation of Hamiltonian systems to the energy dissipation inequality:
\begin{equation}\label{eq:isphs_energy_balance}\begin{split}
    \dot\hamiltonian &=     \underbrace{\partdiff{\hamiltonian}{\bsx}^{\intercal}\mkern-8mu\bsJ\partdiff{\hamiltonian}{\bsx}}_{\substack{\vphantom{\geq}=0}} - \underbrace{\partdiff{\hamiltonian}{\bsx}^{\intercal}\mkern-8mu\bsR\partdiff{\hamiltonian}{\bsx}}_{\substack{\geq0}} + \underbrace{\partdiff{\hamiltonian}{\bsx}^{\intercal}\mkern-8mu\bsG\bsu(t)}_{\substack{s(\bsx, \bsu)}}
        \leq s(\bsx,\bsu).
\end{split}
\end{equation}
In the language of system theory, the system is dissipative
with respect to the supply rate $s(\bsx,\bsu)$. 
The formal definition of dissipativity requires $\hamiltonian\geq0$, which will be ensured later.
Since the supply rate $s(\bsx,\bsu)$ is linear in the input $\bsu$, it vanishes for zero-input $s(\bsx,\bszero)=0$. The inequality thus ensures that the system cannot gain energy in the absence of an external excitation ($\bsu=\bszero$). Besides this physical motivation, the dissipativity of \glspl{PHS} plays a crucial role in stability analysis. 
For the unforced system with ${s=0}$, the Hamiltonian inherently fulfills the decrease condition $\dot\hamiltonian\leq0$.
Thus, by using $\hamiltonian$ as a Lyapunov function, it can be shown that any strict minimum in the Hamiltonian implies a stable equilibrium of the unforced dynamics \cite{vanderschaft2014}.  
Furthermore, the equilibrium is asymptotically stable if the system continuously dissipates energy ($\bsR\succ0$) for states near the minimum.

\revision{
The use of the port-Hamiltonian formulation as a basis for an \gls{ML} model has multiple favorable consequences. 
When learning system dynamics from data, the structure of \cref{eq:isphs_evolution} provides both a physical bias and the opportunity for interpreting the resulting trained model, while  enabling the modeling of a wide range of relevant systems. 
\cref{eq:isphs_energy_balance} highlights the distinctive roles played by the three matrix-valued functions. The structure matrix $\bsJ(\bsx)$ results in conservative dynamics, the dissipation matrix $\bsR(\bsx)$ contributes only to the dissipative dynamics, and the input matrix $\bsG(\bsx)$ captures the energy flow over the system boundaries due to the inputs. This knowledge improves the interpretability of the trained model components. Furthermore, it allows new ways of incorporating prior knowledge into the system. 
For example, the learned system can be constrained to be conservative by fixing $\bsR = \bszero$.
Furthermore, the connection of \gls{PHS} to the decrease condition from Lyapunov's theory provides the opportunity to introduce further desirable stability constraints.
}

\revision{
To illustrate the above concepts with a simple example, consider the dynamics of a three-dimensional spinning rigid body described by Euler's rotation equations.
Expressing the angular velocity vector $\bsomega(t)\in\bbR^3$ in the principal-axis coordinate frame, the evolution equation is given by
\begin{equation}\label{eq:euler_equations}
    \bsI\dot\bsomega + \bsomega\times(\bsI\bsomega) = -\mu\bsomega \quad\forall t>0.
\end{equation}
Here, $\bsI=\diag(I_1, I_2, I_3)$ denotes the inertia matrix and $\mu$ is a damping coefficient.
By introducing
\begin{equation}   \label{eq:rigid_hamiltonian} 
\hamiltonian(\bsomega)=\frac{1}{2}I_1\omega_1^2+\frac{1}{2}I_2\omega_2^2+\frac{1}{2}I_3\omega_3^2,
\end{equation}
which describes the kinetic energy $E$, the system can be rewritten in the form of \cref{eq:isphs_evolution}, with state-dependent \structurematrix $\bsJ$, constant dissipation matrix $\bsR$, and vanishing input matrix $\bsG$:
\begin{equation}
    \begin{bmatrix}
        \dot\omega_1 \\ \dot\omega_2 \\ \dot\omega_3
    \end{bmatrix}
    =
    \Bigg(
    \underbrace{
        \begin{bmatrix}
            0 & -\frac{I_3}{I_1I_2}\omega_3 & \frac{I_2}{I_3I_1}\omega_2 \\
            \frac{I_3}{I_1I_2}\omega_3 & 0 & -\frac{I_1}{I_2I_3}\omega_1 \\
            -\frac{I_2}{I_3I_1}\omega_2 & \frac{I_1}{I_2I_3}\omega_1 & 0
        \end{bmatrix}
    }_{\bsJ(\bsomega)}
    -
    \underbrace{\mu
        \begin{bmatrix}
            \frac{1}{I_1^2} & 0\vphantom{-\frac{I_3}{I_1I_2}\omega_3} & 0 \\
            0 & \frac{1}{I_2^2} & 0\vphantom{-\frac{I_3}{I_1I_2}\omega_3} \\
            0\vphantom{-\frac{I_3}{I_1I_2}\omega_3} & 0 & \frac{1}{I_3^2} \\
        \end{bmatrix}
    }_{\bsR}
    \Bigg)
    \partdiff{\hamiltonian}{\bsomega}.
\end{equation}
In anticipation of the numerical experiment in \cref{sec:spinning_rigid_body}, \cref{fig:spinning_body_energy} shows the predicted energy $E$ from various models trained on trajectories from the spinning body system. 
Those incorporating both, physical biases from the \gls{PHS}-formulation and stability (bPHNN, sPHNN, and sPHNN-LM) perform best, highlighting the importance of these constraints for learning system dynamics.
}

\begin{figure}[t]
    \centering
    \vspace*{-2mm}
    \inputTikzWithExternalization{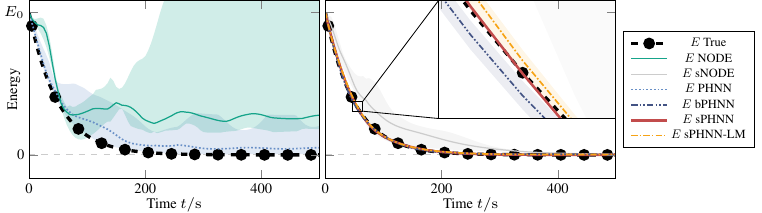}{tikz/spinning_rigid_body/energies.tex}
        \caption{\emph{Spinning rigid body}: \revision{Interquartile mean (lines) and range (shaded regions) of the energy $E$ 
        computed from the predicted states. Left: Models without stability bias; Right: Models with stability bias.}}
    \label{fig:spinning_body_energy}
    \vspace{-0.3cm}
\end{figure}
\section{Stable port-Hamiltonian neural networks}\label{sec:sPHNN}

Our goal is to learn the dynamics of physical systems in the form of \cref{eq:isphs_evolution} with desirable stability properties from observations. 
To this end, we parameterize the components $\hamiltonian$, $\bsJ$, $\bsR$ and $\bsG$ of \gls{PHS} with \glspl{NN}. 

\subsection{Stability of port-Hamiltonian dynamics}

To guarantee the existence of stable equilibria of port-Hamiltonian systems, it is sufficient to enforce the existence of strict local minima in the Hamiltonian. However, these stability implications are only \emph{local}, meaning the size and shape of the corresponding basin of attraction are not constrained by \cref{eq:isphs_evolution}. Small perturbations might be enough to stray far from the stable equilibria. 
Furthermore, the general \gls{PHS} can represent dynamics such as unbounded trajectories, which must be considered unstable from a practical standpoint. 
The following describes the necessary requirements for achieving \emph{global} stability. 
Essentially, the Hamiltonian $\hamiltonian$ is required to be a suitable Lyapunov function, which includes positive definiteness and radial unboundedness.
These properties are achieved by constraining the Hamiltonian to be convex and ensuring the existence of a strict minimum. This approach is similar to the construction of a Lyapunov function by \citet{kolter2019}, but differs in the way the minimum is enforced.
\begin{theorem}\label{prop:stability_requirements}
    Consider the \gls{PHS} \cref{eq:isphs_evolution} in the unforced case $\bsu(t)=\bszero$:
    \begin{equation}\label{eq:isphs_auto_evolution}
        \dot{\bsx} = \left[\bsJ(\bsx) - \bsR(\bsx)\right]\partdiff{\hamiltonian}{\bsx}(\bsx), \quad \text{with}\quad\bsJ=-\bsJ^{\intercal},\,\bsR=\bsR^{\intercal}\succeq0.
    \end{equation}
    Suppose the Hamiltonian $\hamiltonian(\bsx)$ is convex, twice continuously differentiable, and fulfills:
    \begin{align}\label{eq:prop_convex_requirements}
        \hamiltonian(\bszero)=0,&&
        \partdiff{\hamiltonian}{\bsx}\bigg\rvert_{\bsx=\bszero}=\bszero,&&
        \hessian{\hamiltonian}{\bsx}\bigg\rvert_{\bsx=\bszero}\succ0.
    \end{align}
        \revision{
    Then, $\hamiltonian$ is a suitable Lyapunov function for showing stability of the equilibrium at $\bsx(t)=\bszero$, and all solutions are bounded.
    }
    Furthermore, the equilibrium is globally asymptotically stable if $\bsR(\bsx)\succ0$.
\end{theorem}
A proof of \cref{prop:stability_requirements} is provided in \cref{sec:appendix_stability_proof}.
However, its implications also follow a clear physical intuition: 
If a dynamical system continuously loses energy via internal dissipation ($\bsR\succ0$), it will eventually settle into an energy minimum corresponding to an asymptotically stable equilibrium. The convexity requirement, together with \Cref{eq:prop_convex_requirements}, guarantees the existence of a unique strict and global minimum at the origin. Additionally, it ensures that the energy for states at infinity is unbounded, i.e., $\hamiltonian(\bsx)\to\infty$ as $\norm{\bsx}\rightarrow\infty$. Thus, no finite state has more energy than a state at infinity, and all trajectories must be bounded. As a result, all solutions eventually converge to the global energy minimum. 
While convexity is sufficient to ensure global asymptotic stability, it is not a necessary condition. To enhance model expressiveness when needed, techniques such as input warping \cite{kolter2019} can be employed, but were not required in the applications considered in this work.

For conservative systems with $\bsR=\bszero$, energy is not dissipated, and attractive equilibria cannot exist. 
Therefore, a conservative system cannot possess asymptotically stable equilibria. 
Nevertheless, the minimum in the energy still represents a stable equilibrium, and all trajectories stay bounded.
The mixed case, where $\bsR\succeq0$, only allows the same conclusions as the conservative case. 
However, unless the dissipation matrix vanishes in a large region in phase space, it is unlikely that a system with $\bsR\succeq0$ is not globally asymptotically stable.

\subsection{Neural network model architecture}

Having established the theoretical requirements for stability of \gls{PHS}, we now outline how these conditions can be met with a dedicated \gls{NN} model. We call the resulting architecture \mymodel. Its computation graph is illustrated in \cref{fig:sPHNN_architecture}. \cref{fig:random_network} shows the decomposition of a \mymodel with randomly initialized \glspl{NN} into conservative and dissipative dynamics, highlighting the model's physical bias and inherent interpretability.

\begin{figure}[b]
    \centering
    \begin{subfigure}[t]{0.58\textwidth}
        \centering
        \resizebox{\columnwidth}{!}{%
            \inputTikzWithExternalization{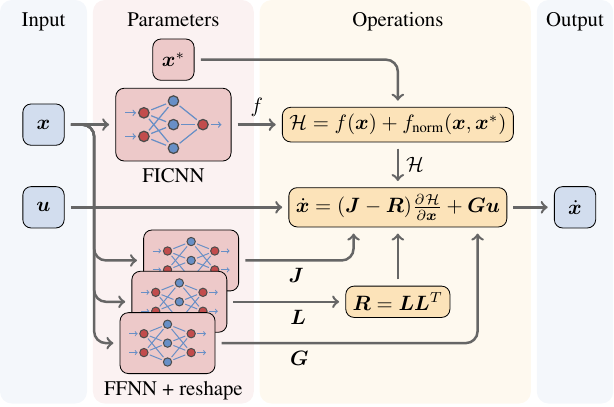}{tikz/computation_graph.tex}
        }
    \caption{Computation graph.}
    \label{fig:sPHNN_architecture}
    \end{subfigure}
    \hfill
    \begin{subfigure}[t]{0.4\columnwidth}
        \centering
        \inputTikzWithExternalization{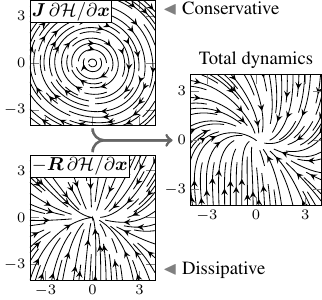}{tikz/interpretability/dynamics_decomposition.tex}
        \caption{Composition of the total dynamics.}
        \label{fig:random_network}
    \end{subfigure}
    \caption{\emph{sPHNN model architecture}: \subref{fig:sPHNN_architecture} Computation graph of \mymodel. The {FICNN} parameterizing $f$ is normalized to obtain $\hamiltonian$ and the outputs of the \glsxtrshortpl{FFNN} for $\bsJ$, $\bsL$, and $\bsG$ are reshaped to be skew-symmetric, lower triangular, and rectangular matrices, respectively. \subref{fig:random_network} Dynamics of a randomly initialized \mymodel for $n=2$, showing the built-in interpretability through the separation into conservative and dissipative dynamics.}
    \end{figure}

\textbf{Hamiltonian}\quad
To fulfill the constraints of \cref{prop:stability_requirements}, the Hamiltonian $\hamiltonian(\bsx)$ is represented by a \gls{FICNN} \cite{amos2017}. \glspl{FICNN} can approximate any convex function \cite{chen2018a}. They are defined by the recurrence relation
\begin{equation}\label{eq:FICNN}\begin{aligned}
    \bsz_1 &= \sigma_0\left(\bsW_0\bsx+\bsb_0\right),\\
    \bsz_{i+1} &= \sigma_i\left(\bsU_i\bsz_i+\bsW_i\bsx+\bsb_i\right) ,\; i=1,2,...,k-1,\\
    f(\bsx) &= z_k,
\end{aligned}\end{equation}
with the pass-through weight matrices $\bsW_i$ that map the input $\bsx$ directly to the activation of the $k-1$ hidden layers, bias vectors $\bsb_i$, and layer outputs $\bsz_i$. 
If the weight matrices $\bsU_i$ have non-negative components, the first activation function $\sigma_0$ is convex, and all subsequent $\sigma_i$ are convex and non-decreasing, then the output $f(\bsx)$ is convex in the input~$\bsx$. 
Since the Hamiltonian should be twice continuously differentiable, the same must hold for the activation functions. One suitable choice is the softplus activation.
The conditions for stability in \cref{eq:prop_convex_requirements} are then fulfilled by normalizing the \gls{FICNN} $f$:
\begin{equation}\label{eq:hamiltonian_normalization}
    \hamiltonian(\bsx) = f(\bsx) \underbrace{- f(\bsx^*) - \partdiff{f}{\bsx}\bigg\vert_{\bsx^*}^{\intercal}(\bsx-\bsx^*)}_{f_\text{norm}(\bsx,\bsx^*)} + \underbrace{\epsilon\norm{\bsx-\bsx^*}^2\vphantom{\partdiff{f}{\bsx}\bigg\vert_{\bsx^*}}}_{f_\text{reg}(\bsx,\bsx^*)}\,.
\end{equation}
Consequently, the Hamiltonian $\hamiltonian$ and its gradient vanish at $\bsx^*$ due to the normalization term $f_\text{norm}$. Moreover, the regularization term $f_\text{reg}$ with an arbitrarily small $\epsilon>0$ ensures positive definiteness of the Hamiltonian's Hessian. 
However, in practice, $f_\text{reg}$ can be omitted since the \gls{FICNN} provides enough bias towards local strict convexity. After training, the stability guarantee can then be recovered by verifying the Hamiltonian's positive definiteness at $\bsx^*$. 
Choosing $\bsx^*=\bszero$ fulfills all requirements of \cref{prop:stability_requirements}. However, the presented approach is more general. As the assumption that the equilibrium is located at the origin is only made for ease of notation, the normalization approach in \cref{eq:hamiltonian_normalization} can position the stable equilibrium anywhere in phase space and is not restricted to the origin. 
\revision{
    If the true equilibrium position is known beforehand, fixing~$\bsx^*$ introduces an additional inductive bias. 
    While this information can often be obtained from prior physical knowledge, it may not always be available. 
    In such cases, the equilibrium can be inferred directly from data by treating~$\bsx^*$  as a trainable parameter during optimization.
}

\textbf{Structure, dissipation and input matrices}\quad
The matrix-valued functions $\bsJ(\bsx),\,\bsR(\bsx)$ and $\bsG(\bsx)$ are parameterized by \glspl{FFNN}, where the outputs are suitably reshaped into matrices. 
As $\bsJ$ is constrained to be skew-symmetric, the respective \gls{FFNN}'s output vector is mapped to the space of skew-symmetric matrices via a linear transformation. 
Similarly, the positive definiteness (or semi-definiteness) of $\bsR$ is ensured via the Cholesky factorization $\bsR=\bsL\bsL^{\intercal}$, where the reshaped \gls{FFNN} output $\bsL$ is a lower triangular matrix with positive (or non-negative) elements along the main diagonal.
For some applications, $\bsJ, \,\bsR$ and $\bsG$ are independent of $\bsx$, and thus, learning constant matrices is sufficient. 
In particular, this is the case when the conservative dynamics of the system under consideration are Hamiltonian in the given coordinates. 
Furthermore, by fixing $\bsJ$ to the symplectic matrix (and choosing $\bsR=\bsG=\bszero$), the \mymodel can be reduced to a Hamiltonian \gls{NN}.
 
\subsection{Training stable port-Hamiltonian neural networks}
There are two fundamentally different approaches for training time-continuous dynamic models such as \NODEs and \mymodels: \emph{derivative} and \emph{trajectory fitting}. 
The former directly compares the predicted state derivatives $\dot\bsx$ from \cref{eq:isphs_evolution} to the true ones. 
While this approach is efficient, it requires the true time derivatives to be available as training data, which is typically not the case for real-world measurement data. Furthermore, derivative fitting is not applicable when augmented states are used \cite{dupont2019}, i.e., when $\bsx$ is padded with additional dimensions compared to the observable states to enrich the representable dynamics. 
Further, a model proficient in predicting $\dot\bsx$ does not necessarily produce accurate predictions for trajectories $\bsx(t)$, as small derivative errors can accumulate exponentially when integrating an ODE \cite{hairer2009}.
On the other hand, trajectory fitting directly compares integrated model trajectories $\bsx(t)$ with ground-truth trajectories, ensuring long-term accuracy.
Here, the optimization process necessitates either the propagation of the loss gradients through the ODE solver or the use of the adjoint sensitivity method \cite{chen2018}, which computes the gradients by solving a modified ODE. In both cases, trajectory fitting incurs significant computational costs due to gradient propagation and model integration, making trajectory fitting much slower than derivative fitting.

A distinct advantage of \mymodels is the favorable influence the stability constraint exerts on the training process when using trajectory fitting. 
\CapNODEs and similar architectures can suffer from exploding gradients, where small parameter changes lead to large loss gradients. 
This can slow the training process or prevent it from converging. 
A potential cause is the crossing between the boundaries of local basins of attraction during training, leading to significant differences in long-term behavior \cite{pascanu2013}. 
Since \mymodels have only one global basin of attraction, this source of training difficulties is eliminated. 
In our numerical experiments, we observe good convergence of the training error of \mymodels where \NODEs struggle without the use of regularization terms or gradient clipping.

\section{Experiments}\label{sec:experiments}

We evaluate the proposed model on multiple example problems, including measurement data from a real physical system and multiphysics simulation data.
The performance of \mymodels is compared to that of \NODEs, \PHNNs, and \glspl{bPHNN}. 
The latter two are \gls{PHS}-based \gls{NN} models that do not enforce or relax the convexity constraint of \mymodels, and thus do not guarantee global stability. 
In \PHNNs, the Hamiltonian is modeled by an unconstrained \gls{FFNN}, whereas in \glspl{bPHNN}, it is constructed to be positive and radially unbounded, which ensures bounded trajectories. 
\revision{The latter architecture is intentionally designed to resemble} OnsagerNet \cite{yu2021}, but with the key difference that it supports arbitrary time-dependent inputs, making it suitable for our experimental setting.
For examples that purely rely on derivative fitting, we further include \glspl{sNODE} \cite{kolter2019}. 
\revision{Using trajectory fitting, we could not achieve convergence of \glspl{sNODE}, excluding them from most experiments.}
This is likely due to difficulties caused by discontinuous dynamics stemming from the projection-based stability enforcement, as also mentioned by \cite{schlaginhaufen2021}.
\revision{Finally, we include a variant of \mymodels with learnable minima (\mymodel-LM). 
In contrast to \mymodels, which have $\bsx^*$ fixed to an a-priori known equilibrium position, \mymodel-LMs have to infer the equilibrium position from the training data and start out with a randomly initialized $\bsx^*$.}
All models are trained by minimizing the \gls{MSE} using the ADAM optimizer \cite{kingma2015}. 
Model hyperparameters \revision{and computational costs} are listed in \cref{sec:experiment_hyperparameters_appendix}.
\revision{Predictions are generated by rollout with a numerical integrator using the Runge-Kutta scheme Tsit5, with adaptive step size and a known input function $\bsu(t)$.}
Code and model weights are available on GitHub\footnote{\giturl}.

\subsection{Spinning rigid body}\label{sec:spinning_rigid_body}

We generate data of a damped spinning rigid body by numerically integrating Euler's rotation equations, \revision{which were already introduced in \cref{sec:phs}, \cref{eq:euler_equations}.}
Here, accurately capturing the dynamics of the angular velocity vector $\bsx=\bsomega$ with a \gls{PHS} requires a state-dependent \structurematrix~$\bsJ(\bsx)$.
We train ten instances per model type using derivative fitting with $(\bsomega, \dot\bsomega)$ pairs sampled from the trajectories.
\cref{fig:spinning_body_energy} shows the rotational energy $E(\bsomega)$ computed from model predictions~$\bsomega$ via \cref{eq:rigid_hamiltonian}.
The \PHNN, \gls{NODE}, and \gls{sNODE} models show deviations from ground truth, with the former two not converging to zero.
Furthermore, the \NODEs show unphysical behavior with at times increasing energy and exhibit high variance between instances. 
In contrast, \mymodels and \glspl{bPHNN} match the true energy well, highlighting the importance of providing both stability and physical biases.
\revision{
The \mymodel-LM variant recovered the actual equilibrium position accurately, with an interquartile mean distance between the true and inferred minimum across all instances of \num{0.010}. As a result, its predictive performance is nearly identical to that of the sPHNN with fixed~$\bsx^*$.
}

\subsection{Cascaded tanks}\label{sec:cascaded_tanks}

Next, the \mymodel model is evaluated using the cascaded tanks dataset \cite{schoukens2020, schoukens2016}. 
It contains measurements of a physical fluid level control system consisting of a pump and two tanks, see \cref{sec:cascaded_tanks_appendix} for details. 
The training data consists of a single trajectory with 1024 pairs of pump voltages $u$ and water level measurements $y$ of the lower tank.
We model the system using a two-dimensional state vector $\bsx=[x_1, x_2]^{\intercal}$, where $y=x_2$.  
Assuming the augmented state $x_1$ describes the water level in the upper tank, it is straightforward to determine the equilibrium of the autonomous system. With zero input, the tanks will eventually drain; therefore, the stable equilibrium is given by $\bsx^*=\bszero$. We thus fix the minimum of the \mymodel's Hamiltonian to the origin. 
To investigate the effects of a potential lack of this knowledge, we also train the \mymodel-LM variant. 
Its minimum location $\bsx^*$ is randomly initialized and optimized during training.
Apart from this, both \mymodel versions are identical and use a constant fixed symplectic matrix as the \structurematrix, as well as constant but learnable dissipation and input matrices. 
The same structure applies to PHNN and bPHNN.

\begin{figure}[t]
    \vspace*{-0.1cm}
    \centering
    \begin{subfigure}[b]{0.5\columnwidth}
        \centering
        \resizebox{\columnwidth}{!}{%
            \inputTikzWithExternalization{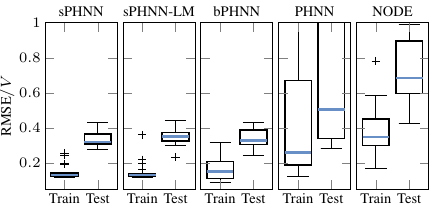}{tikz/cascaded_tanks/boxplot_short_labels.tex}
        }
        \caption{Training and test \glsxtrshort{RMSE}.}
        \label{fig:cascaded_tanks_rmse}
    \end{subfigure}
    \begin{subfigure}[b]{0.49\columnwidth}
        \centering
        \inputTikzWithExternalization{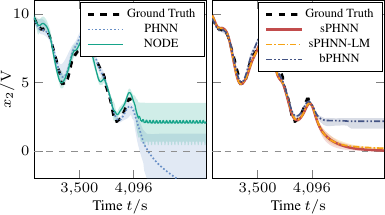}{tikz/cascaded_tanks/relaxation_all_in_two.tex}
        \vspace*{-1mm}
        \caption{Relaxation.}
        \label{fig:cascaded_tanks_relaxation}
    \end{subfigure}
    \vspace*{-4mm}
    \caption{\emph{Cascaded tanks}: \subref{fig:cascaded_tanks_rmse} \Glsxtrshort{RMSE} of the models on training and test trajectory. \subref{fig:cascaded_tanks_relaxation}: Predictions for the extended test trajectory. At $t=\qty{4096}{\second}$, the pump is turned off.  Lines correspond to the interquartile mean and shaded areas represent the interquartile range of the predictions from the \num{20} model instances.}
    \label{fig:cascaded_tanks}
    \vspace*{-4mm}
\end{figure}

We train \num{20} instances per model type and show a statistical evaluation of the \glspl{RMSE} in \cref{fig:cascaded_tanks_rmse}. 
Overall, \mymodel, \mymodel-LM, and \gls{bPHNN} achieve comparable accuracy, while \NODE and \PHNN perform worse on average across both trajectories. In addition to their better performance, the models with stability bias exhibit lower variance across model instances, indicating more consistent results.
To evaluate the models' extrapolation capabilities, we extend the evaluation data with \qty{400}{s} of zero-input. The resulting predictions are depicted in \Cref{fig:cascaded_tanks_relaxation}. 
Assuming~perfectly calibrated sensors, the true system response is an eventual zero output as both tanks drain. 
By construction, this behavior is guaranteed for \mymodels. 
The \num{20} \mymodel-LM instances reliably identified the systems equilibrium to lie between \qtyrange{-0.196}{0.432}{\volt}, which aligns with physical intuition.
In contrast the \gls{bPHNN}, \NODE and \PHNN models exhibit poor physical extrapolation, predicting equilibria far from zero, diverging water levels, or oscillatory behavior.
In an additional experiment presented in \cref{sec:cascaded_tanks_appendix}, the \mymodel models are also shown to have successfully learned the hard saturation effect that occurs when a water tank overflows --- a task at which both \NODE and \PHNN fail.

\subsection{Thermal food processing surrogate}\label{sec:chicken_data}

With the applicability of \mymodels to real-world data established, this section focuses on one of the promising applications of data-driven dynamic system identification: the efficient construction of surrogate models.
Here, data from a conjugate heat transfer simulation of a convection oven coupled with a soft matter model for meats from \citet{kannapinn2022,kannapinn2024} is employed. 
The data consists of trajectories representing the temperature histories $T_A$ and $T_B$ at two probe points within the meat. These result from a predefined excitation signal $T_{\text{oven}}$, which controls the oven temperature. 
The surrogate's task is to predict the probe temperatures given the oven temperature as input. 

\begin{figure}[t]
    \centering
    \begin{subfigure}[b]{0.49\textwidth}
        \centering
        \inputTikzWithExternalization{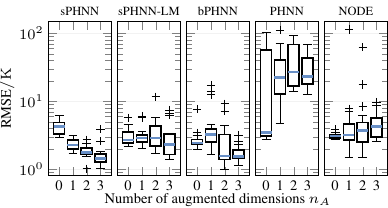}{tikz/chicken_data/rmse_boxplot.tex}
        \vspace*{1mm}
        \caption{Test \glspl{RMSE} for $n_D=\num{2}$ training trajectories.}
        \label{fig:chicken_data_rmse}
    \end{subfigure}
        \begin{subfigure}[b]{0.49\columnwidth}
        \centering
        \inputTikzWithExternalization{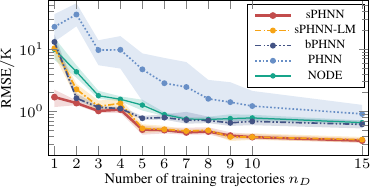}{tikz/chicken_data/test_rmse_per_num_dat.tex}
        \vspace*{1mm}
        \caption{Test \glspl{RMSE} for $n_A=\num{3}$ augmented dimensions.}
        \label{fig:chicken_data_rmse_per_num_dat}
    \end{subfigure}\\
    \vspace*{1mm}
    \begin{subfigure}[t]{\columnwidth}
        \centering
        \resizebox{\linewidth}{!}{%
            \inputTikzWithExternalization{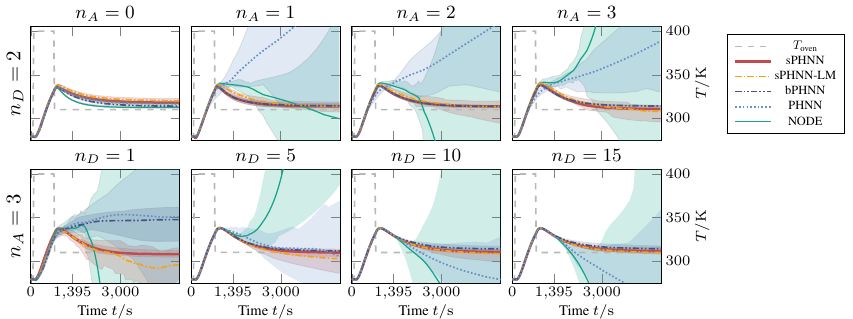}{tikz/chicken_data/relaxation_dim_and_aug.tex}
        }
        \vspace*{-3mm}
        \caption{Model predictions for different combinations of augmented dimensions and training data.}
        \label{fig:chicken_data_best_predictions_relaxation}
    \end{subfigure}
    \caption{\emph{Thermal food processing surrogate}: \subref{fig:chicken_data_rmse} and \subref{fig:chicken_data_rmse_per_num_dat} \glsxtrshortpl{RMSE} evaluated on \num{15} test trajectories for various numbers of augmented dimensions $n_A$ and training trajectories $n_D$. 
    \subref{fig:chicken_data_best_predictions_relaxation} Interquartile mean and range of the $T_A$ predictions for a custom test case. \emph{Top row}: Varying number of augmented dimensions with fixed number $n_D=\num{2}$ of training trajectories. \emph{Bottom row}: Varying number of training trajectories with fixed number $n_A=\num{3}$ of augmented dimensions. The time $t=\qty{1395}{\second}$ marks the length of the training trajectories.}
    \label{fig:chicken_data}
    \vspace*{-1mm}
\end{figure}

We fit the models to initially $n_D=\num{2}$ trajectories consisting of \num{280} samples each (see \cref{sec:chicken_data_appendix}). 
Per model, we vary the dimensionality of the state ${\bsx\in\bbR^n}$ with~$n=2+n_A$ by adding \numrange{0}{3} augmented dimensions $n_A$.
The minimum $\bsx^*$ of the \mymodels' Hamiltonian is set to the point in phase space corresponding to the ambient temperature, as motivated by physical intuition about the stable thermal equilibrium.
The initial values of the augmented states are fixed to zero.
In total, \num{20} instances are trained per model type and augmented dimension. 

The \glspl{RMSE} evaluated on 15 test trajectories (see \cref{sec:chicken_data_appendix}) are depicted in \cref{fig:chicken_data_rmse}. 
The \mymodel predictions tend to improve with more augmented dimensions, achieving errors comparable to those obtained by \citet{kannapinn2022} using commercial software. 
This appears intuitive, as additional hidden state variables may be required to accurately describe the evolution of~$T_A$ and~$T_B$ that results from the multiphysics, conjugate heat transfer process, see also \cite{dupont2019}.
However, for \glspl{bPHNN} \revision{and \mymodel-LMs} no clear trend emerges.
The \NODEs and \PHNNs even tend to deteriorate in performance with more augmented dimensions, likely due to instabilities in the learned dynamics.
This becomes apparent in \cref{fig:chicken_data_best_predictions_relaxation} (top row), which depicts the models' predictions for an extended test case. 
While the \mymodels consistently converge to temperatures close to thermal equilibrium, the predictions of \NODEs and \PHNNs tend to become unstable when one or more augmented dimensions are used.
Although all models improve with more training data (\cref{fig:chicken_data_rmse_per_num_dat}), these instabilities persist when extrapolating in time (\cref{fig:chicken_data_best_predictions_relaxation}, {bottom row} $t>\qty{1395}{\second}$).
\revision{
    In contrast, the stability constraint of the \mymodels and \mymodel-LMs allows safe use of the additional flexibility provided by a higher-dimensional state space, avoiding the pitfalls observed in the unconstrained models.
    Given sufficient training data, the \mymodel-LM variant performs on par with the \mymodel, whereas in the scarce-data regime its errors increase and are comparable to those of the \glspl{bPHNN}.
    Although the \glspl{bPHNN} also remain stable, they are consistently outperformed by the \mymodels for $n_A=\num{3}$, regardless of dataset size, and especially when only a single training trajectory is available $n_D=\num{1}$, see \cref{fig:chicken_data_rmse_per_num_dat} and \cref{fig:chicken_data_best_predictions_relaxation} (bottom left).
}

\revision{
    To assess the models' robustness to noise, we conducted an additional experiment using $n_A=\num{3}$ augmented dimensions and $n_D=\num{2}$ training trajectories, introducing zero-mean Gaussian noise for both the input $T_{\text{oven}}$ and the outputs $T_A$ and $T_B$. 
    The noise amplitudes were varied between \qty{5}{\percent} and \qty{25}{\percent} of the original signals' standard deviations. 
    \Cref{tab:chicken_data_noise_rmses} reports the average \gls{RMSE} over 20 model instances per model type, evaluated on noise-free test data.
    To isolate the effect of noise and eliminate generalization error, we also evaluated the \gls{RMSE} on the noise-free training data.
    All models demonstrate a degree of robustness to noise, since the integration step involved in trajectory fitting acts as a low-pass filter, effectively smoothing the learned dynamics. 
    As a result, high-frequency Gaussian noise is attenuated, which helps prevent overfitting and improves generalization.
    Nevertheless, the \mymodel consistently outperforms the baselines on both the training and test sets.
}

\begin{table}
    \vspace*{-2mm}
    \caption{\revision{Median \gls{RMSE} of \num{20} instances per models using $n_A=\num{3}$ augmented dimensions, trained on $n_D=\num{2}$ noisy trajectories and evaluated on noise-free training and test sets.}}
    \label{tab:chicken_data_noise_rmses}
    \vspace*{2mm}
    \centering
    \resizebox{0.9\linewidth}{!}{%
    \begin{tabular}{lrrrrrrr}
        \toprule
        & \multicolumn{3}{c}{Training} & & \multicolumn{3}{c}{Test}\\
        \cmidrule(r){2-4}
        \cmidrule(r){6-8}
        Model & \qty{5}{\percent} noise & \qty{15}{\percent} noise & \qty{25}{\percent} noise & & \qty{5}{\percent} noise & \qty{15}{\percent} noise & \qty{25}{\percent} noise\\ 
        \midrule
        sPHNN & \bf{0.332} & \bf{0.854} & \bf{1.354} &&  \bf{1.170} &  \bf{2.595} &  \bf{3.177}\\
        sPHNN-LM & 0.638 & 1.096 & 1.622 && 2.420 & 2.888 & 3.677 \\
        bPHNN & 0.500 & 1.049 & 2.207 &&  1.572 &  2.828 &  5.231\\
        PHNN  & 2.828 & 6.750 & 5.779 && 28.233 & 25.631 & 24.224\\
        NODE  & 0.993 & 1.974 & 2.423 &&  3.577 &  3.482 &  4.362\\
        \bottomrule
    \end{tabular}
    }
    \vspace*{-4mm}
\end{table}

\subsection{Additive manufacturing surrogate}\label{sec:thermal_field_data}

Finally, we evaluate the applicability of \mymodels to higher-dimensional state spaces and build a \gls{ROM} as a surrogate of a 3D \gls{PDE} problem. 
To this end, we generate field data by numerically solving the heat conduction equation for a moving heat source on a cuboid domain with convective and radiative thermal boundary conditions. 
The \glsxtrlong{FE} simulation is designed to model the evolution of the temperature field of a metal additive manufacturing process, see \cite{kannapinn2024a} for details.
In total, \num{25} trajectories are obtained by varying the heat source speed~$v$ from \qtyrange{10}{20}{\milli\metre\per\second} and power~$Q$ from \qtyrange{300}{500}{\watt}.
All trajectories span \qty{20}{\second}, with a non-zero heat source lasting for the initial \qtyrange{2.5}{6}{\second}, depending on~$v$.
However, directly using the discretized field data would lead to very high-dimensional models and render the training and inference inefficient. 
We thus apply a \gls{POD} to map the temperature and heat source fields onto \numsvdmodes-dimensional latent spaces, respectively (see \cref{sec:thermal_field_data_appendix} for details).
The models are trained on the latent representations of the two trajectories with $(v,Q)=(\qty{10}{\milli\metre\per\second}, \qty{300}{\watt})$ and $(\qty{20}{\milli\metre\per\second}, \qty{500}{\watt})$ using derivative fitting and subsequent fine-tuning with trajectory fitting.

\begin{figure}[t!]
    \vspace*{-1mm}
    \centering
    \resizebox{\linewidth}{!}{%
        \inputTikzWithExternalization{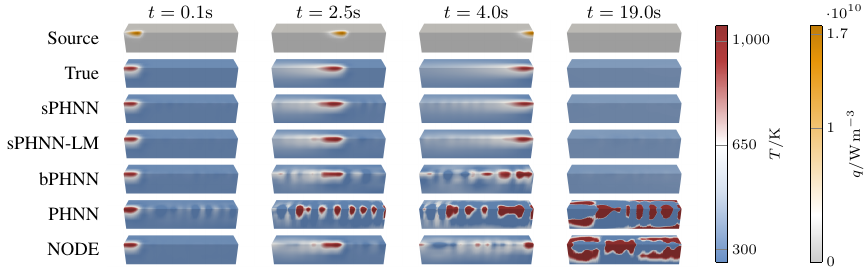}{tikz/thermal_ded/mesh_predictions.tex}
    }
    \vspace*{-3mm}
    \caption{Heat source field and temperature field predictions on the cuboid domain for a test case with $v=\qty{12.5}{\milli\metre\per\second}$, $Q=\qty{400}{\watt}$. The instances selected for this evaluation resulted in the median test error for the corresponding model type. Colors are clipped to remain in the legend's range.}
    \label{fig:ded_mesh_predictions}
    \vspace*{-3mm}
\end{figure}

\Cref{fig:ded_mesh_predictions} depicts the heat source field (top row) and true and predicted temperature fields from the different model types. 
While the \mymodel's total test \gls{RMSE} peaks below \qty{30}{\kelvin} (see \cref{sec:thermal_field_data_appendix}, \cref{fig:ded_mesh_rmses_over_time}) and decreases as the system approaches thermal equilibrium, the errors for \NODE and \PHNN increase rapidly, driven by instabilities in the learned dynamics.
These instabilities are likely caused by the sparsity of the training data, which leaves gaps in the phase space, allowing unstable dynamics to emerge. 
The stability constraint of \mymodel \revision{and \mymodel-LM} mitigates these issues, ensuring robust performance even with scarce training data, while outperforming the \gls{bPHNN} thanks to their stronger inductive bias.

\section{Conclusion}\label{sec:conclusion}
\glsreset{sPHNN}
This work proposes \mymodels for accurate, robust, and reliable identification of stable nonlinear dynamics.
They leverage the port-Hamiltonian framework, which provides a basis for model interpretation and ensures adherence to the laws of thermodynamics in terms of energy conservation or dissipation balances with respect to a learned Hamiltonian energy function.
\revision{Additionally, the approach ensures global asymptotic stability of the identified dynamics, without requiring projection, by constraining the Hamiltonian to be a convex, positive definite Lyapunov function.}
While this limits \mymodels to modeling systems that appear globally stable within the region of interest, it provides a strong inductive bias that enables robust learning and reasonable generalization from sparse data. 
We demonstrate the approach's viability using data from synthetic examples, real-world measurements, and complex multiphysics simulation models. 
The resulting models, built with small neural network architectures, exhibit low variance across instances, accurate prediction, and robust generalization. Furthermore, the proposed method can successfully exploit augmented dimensions where unconstrained alternative models struggle with instabilities or data scarcity.
\revision{While the approach can identify the equilibrium point from data, we find that providing the equilibrium as a prior can further enhance prediction accuracy and reduce data requirements.}
Future work will extend the approach to multistable systems by allowing Hamiltonians with multiple minima, while retaining favorable stability properties.

\begin{ack}
The authors acknowledge the financial support provided by the Deutsche Forschungsgemeinschaft (DFG, German Research Foundation, research grant number 492770117 and the CRC/TRR 361 CREATOR), the Hessian Center for Artificial Intelligence (Hessian.AI, project EnDyRo), and the Graduate School Computational Engineering at TU Darmstadt.
Furthermore, we would like to thank An Thai Le, Daniel Palenicek, and Joe Watson from TU Darmstadt (IAS) for their constructive feedback on the manuscript.

\end{ack}

\printbibliography

\clearpage 
\appendix
\crefalias{section}{appendix}
\counterwithin{equation}{section}
\counterwithin{figure}{section}

\section{Proof of stability}\label{sec:appendix_stability_proof}
This section provides a proof of \cref{prop:stability_requirements}, which describes the essential requirements for achieving stable dynamics with \glsxtrlongpl{PHS}. Due to the relevance of convexity to the following discussion, the following Lemma first provides some widely known equivalent formulations of convexity for scalar-valued multivariate functions. 
\begin{lemma}\label{lem:convexity}
    Consider a twice continuously differentiable function $f:\bbR^n\rightarrow\bbR,\,\bsx\mapsto f(\bsx)$.
                If $f$ is convex, the following statements hold:
    \begin{enumerate}
        \item $\lambda f(\bsx_1) + (1-\lambda)f(\bsx_2) \geq f(\lambda\bsx_1 + (1-\lambda)\bsx_2)\quad\forall \bsx_1,\bsx_2\in\bbR^n,\, \lambda\in[0,1]$,
        \item $\partdiff{f}{\bsx}\big\rvert_{\bsx_1}^{\intercal}(\bsx_2-\bsx_1)\leq f(\bsx_2)-f(\bsx_1)\quad\forall \bsx_1,\bsx_2\in\bbR^n$,
                        \item $\hessian{f}{\bsx}(\bsx)\succeq0\quad\forall\bsx\in\bbR^n$.
    \end{enumerate}
\end{lemma}
For a proof, the reader is referred to \cite{boyd2023}. 
\revision{Note that for strict convexity, the inequalities in statements 1 and 2 must be strict for all $\bsx_1\neq\bsx_2$, and the Hessian in statement 3 must be positive definite.}
With the notion of convexity established, attention can now be directed to the proof of \cref{prop:stability_requirements}. For convenience, the statement of the theorem is repeated verbatim below:
\begin{theorem}\label{prop:stability_requirements_appendix}
    Consider the \gls{PHS} \cref{eq:isphs_evolution} in the unforced case $\bsu(t)=\bszero$:
    \begin{equation}\label{eq:isphs_auto_evolution_appendix}
        \dot{\bsx} = \left[\bsJ(\bsx) - \bsR(\bsx)\right]\partdiff{\hamiltonian}{\bsx}(\bsx), \quad \text{with}\quad\bsJ=-\bsJ^{\intercal},\,\bsR=\bsR^{\intercal}\succeq0.
    \end{equation}
    Suppose the Hamiltonian $\hamiltonian(\bsx)$ is convex, twice continuously differentiable, and fulfills:
    \begin{align}\label{eq:prop_convex_requirements_appendix}
        \hamiltonian(\bszero)=0,&&
        \partdiff{\hamiltonian}{\bsx}\bigg\rvert_{\bsx=\bszero}=\bszero,&&
        \hessian{\hamiltonian}{\bsx}\bigg\rvert_{\bsx=\bszero}\succ0.
    \end{align}
    Then, the system in \cref{eq:isphs_auto_evolution_appendix} has a stable equilibrium at $\bsx(t)=\bszero$, and all solutions are bounded. Furthermore, the equilibrium is globally asymptotically stable if $\bsR(\bsx)\succ0$.
\end{theorem}
\begin{proof}
    Since the gradient of the Hamiltonian vanishes and its Hessian is positive definite \revision{at $\bsx=\bszero$}, it has a strict local minimum at the origin. This implies the existence of an $R>0$ such that 
    \begin{equation}\label{eq:proof_minimum_def}
    \hamiltonian(\bsx)>\hamiltonian(\bszero)=0 \quad \forall 
       \bsx\in\{\norm{\bsx}\leq R,\,\bsx\neq\bszero\}.
    \end{equation}
    Consider the function $g(\bsx)=\frac{a}{r}\norm{\bsx}$ for some $r\in(0,R]$ with $a=\inf_{\norm{\bsx}=r}\hamiltonian(\bsx)>0$. We claim that $g$ is a lower bound for $\hamiltonian$, i.e. $g(\bsx)\leq\hamiltonian(\bsx)$ for all $\norm{\bsx}>r$.

    Suppose there existed an $\tilde\bsx\in\bbR^n$ that violated the claim, that is $\norm{\tilde\bsx}>r$ and $g(\tilde\bsx)>\hamiltonian(\tilde\bsx)$. Due to the convexity of $\hamiltonian$, it holds
    \begin{equation}
    \lambda\hamiltonian(\tilde\bsx)+(1-\lambda)\hamiltonian(\bszero)\geq\hamiltonian(\lambda\tilde\bsx+(1-\lambda)\cdot\bszero) \quad\forall \lambda\in[0,1]
    \end{equation}
    and thus
    $\lambda g(\tilde\bsx)>\lambda\hamiltonian(\tilde\bsx)\geq\hamiltonian(\lambda\tilde\bsx).$
    Choosing $\lambda=\frac{r}{\norm{\tilde\bsx}}$ and applying the definition of $g$ we obtain
    \begin{equation}
    \frac{r}{\norm{\tilde\bsx}}g(\tilde\bsx) = a > \hamiltonian\bigg(\frac{r}{\norm{\tilde\bsx}}\tilde\bsx\bigg)\geq a.
    \end{equation}
    This is a contradiction, and therefore $g$ is a lower bound.

    With the bound established, showing that $\hamiltonian$ is radially unbounded is straightforward. 
    For any path $\bsx(t)$ with $\lim_{t\rightarrow\infty}\norm{\bsx(t)}=\infty$, once $\norm{\bsx(t)}>r$ we have
    \begin{equation}
    \lim_{t\rightarrow\infty}\hamiltonian(\bsx(t)) \geq \lim_{t\rightarrow\infty} g(\bsx(t)) = \infty.
    \end{equation}
    \revision{We would like to remark that the same conclusion could alternatively be reached by showing that $\hamiltonian$ has at least one bounded sublevel set due to its strict minimum and then concluding radial unboundedness (i.e., coercivity) using Corollary 8.7.1 in \cite{rockafellar1970} together with the standard equivalence for proper closed convex functions. However, the presented approach using a lower bound can directly be extended to show global positive definiteness of $\hamiltonian$.
    For this,} note that for all $\norm{\bsx}>r$ the lower bound implies 
    $0<g(\bsx)\leq\hamiltonian(\bsx).$
    Together with \cref{eq:proof_minimum_def}, this implies that $\hamiltonian$ is indeed positive definite for all $\bsx\in\bbR^n$.

    From the port-Hamiltonian structure of \cref{eq:isphs_auto_evolution_appendix} we immediately obtain the decrease condition $\dot\hamiltonian\leq 0$ (see \cref{eq:isphs_energy_balance}).
    Using the $\hamiltonian$ as a Lyapunov function, we obtain all requirements to conclude \emph{local} stability of the equilibrium.
    Furthermore, the decrease condition implies that a solution starting at $\bsx^*$ cannot leave the sublevel set $\Omega_\hamiltonian=\{\bsx\in\bbR^n\vert\hamiltonian(\bsx)\leq\hamiltonian(\bsx^*)\}$.
    Since $\Omega_\hamiltonian\subseteq\Omega_g=\{\bsx\in\bbR^n\vert g(\bsx)\leq\hamiltonian(\bsx^*)\}$ and $\Omega_g$ is clearly bounded, all solutions of \cref{eq:isphs_auto_evolution_appendix} must be bounded as well.

    Finally, to show that \emph{global} asymptotic stability follows from $\bsR\succ0$, it remains to be demonstrated that $\dot\hamiltonian<0$ for all $\bsx\neq\bszero$. 
    This is implied by \cref{eq:isphs_energy_balance} if $\partdiff{\hamiltonian}{\bsx}\neq\bszero$ for all $\bsx\neq\bszero$. Proving that the gradient of $\hamiltonian$ vanishes only at the origin is done by contradiction. Assume that there exists an $\tilde\bsx\neq\bszero$ such that $\partdiff{\hamiltonian}{\bsx}\big\rvert_{\tilde\bsx}=\bszero$. Furthermore, let $\bsx^*=\bszero$. 
    Since $\hamiltonian$ is convex it follows (see \cref{lem:convexity}) that
    \begin{align}
        \hamiltonian(\bsx^*)-\hamiltonian(\tilde\bsx) \geq \partdiff{\hamiltonian}{\bsx}\bigg\rvert_{\tilde\bsx}^{\intercal}(\bsx^*-\tilde\bsx) = 0,\\
        \hamiltonian(\tilde\bsx)-\hamiltonian(\bsx^*) \geq \partdiff{\hamiltonian}{\bsx}\bigg\rvert_{\bsx^*}^{\intercal}(\tilde\bsx-\bsx^*) = 0.
    \end{align}
    This implies $\hamiltonian(\bsx^*)=\hamiltonian(\tilde\bsx)=0$. However, since $\hamiltonian$ is globally positive definite and thus only vanishes at the origin, we have $\bsx^* = \tilde\bsx$. This is a contradiction as $\tilde\bsx\neq\bszero$, which concludes the proof.
\end{proof}

\section{Experiment hyperparameters} \label{sec:experiment_hyperparameters_appendix}
In each experiment presented within this manuscript, the different models were trained for the same number of steps using identical learning rates. The mean squared error loss function and the ADAM \cite{kingma2015} optimizer were used throughout.
The model and training hyperparameters, along with the number of model instances trained for statistical evaluation, are provided in \cref{tab:hyperparameters}. 
Throughout the experiments presented in \crefrange{sec:spinning_rigid_body}{sec:chicken_data}, all \glspl{NN} consisted of two hidden layers, each containing \num{16} neurons. The network width was increased to \num{32} neurons for the additive manufacturing surrogate in \cref{sec:thermal_field_data}. All network weights were initialized using the Glorot uniform \cite{glorot2010} scheme. 

\begin{table}[b]
\def\statedependent{s.}
\def\constant{c.}
\caption{Hyperparameters per experiment. TF and DF refer to trajectory and derivative fitting, respectively. The number of instances trained per model type is listed in the "\# inst." column, and $n$ and $m$ denote the number of dimensions of state and input. The final three columns list the parametrization of the matrices for \mymodel, \revision{\mymodel-LM}, \gls{bPHNN}, and \PHNN as either state-dependent matrices (\statedependent) or constant matrices (\constant).}
\label{tab:hyperparameters}
\begin{center}
\begin{small}
\begin{tabular}{lcccccccc} \toprule
Experiment & \# training steps & learning rate & \# inst. & $n$ & $m$ & \multicolumn{3}{c}{matrices} \\
&&&&&&$\bsJ$ & $\bsR$ & $\bsG$ \\
\midrule
\hyperref[fig:random_network]{Randomly initialized model} & -                   &  -            & \num{1}  & \num{2} & - & \statedependent & \statedependent & - \\
\hyperref[sec:spinning_rigid_body]{Spinning rigid body}        & \num{50000} DF      & \num{1d-3}    & \num{10} & \num{3} & - & \statedependent & \constant & -\\
\hyperref[sec:cascaded_tanks]{Cascaded tanks}             & \num{50000} TF      & \num{5d-4}    & \num{20} & \num{2} & \num{1} & \constant & \constant & \constant\\
\hyperref[sec:chicken_data]{Food processing surrogate}  & \num{30000} TF      & \num{1d-4}    & \num{20} & \numrange{2}{5} & \num{1} & \constant & \constant & \constant\\
\multirow{2}{*}{\makecell[l]{\hyperref[sec:thermal_field_data]{Additive manufacturing}\\\hphantom{m}\hyperref[sec:thermal_field_data]{surrogate}}}  & \num{20000} DF      & \num{1d-3}    & \multirow{2}{*}{\num{20}} & \multirow{2}{*}{\num{40}} & \multirow{2}{*}{\num{40}} & \multirow{2}{*}{\constant} & \multirow{2}{*}{\constant} & \multirow{2}{*}{\constant}\\
                   &  and \num{10000} TF & \num{1d-5}    &          &  &\\
\bottomrule
\end{tabular}
\end{small}
\end{center}
\vskip -0.1in
\end{table}

The \gls{PHS}-based models (\mymodel, \mymodel-LM, \PHNN, \gls{bPHNN}) share the same structure, differing only in how the Hamiltonian $\hamiltonian$ is parametrized.
Specifically, the \mymodels use an \gls{FICNN} with a normalization term, as described in \cref{sec:sPHNN}. 
\revision{The only difference for the \mymodel-LM variant is that the minimum $\bsx^*$ of the Hamiltonian is learned as an additional parameter.}
The \PHNNs, by contrast, use an unconstrained \gls{FFNN} without any normalization. 
Finally, the \glspl{bPHNN} ensures positivity and radially unboundedness of the Hamiltonian by adopting the method described by \citet{yu2021} for the OnsagerNets Energy function. Here, the output of an unconstrained network is squared and combined with a quadratic regularization term. However, unlike OnsagerNet, we used standard \glspl{FFNN} instead of ResNets.

The models were trained on a Windows machine with Intel i7-13700K (13th Generation) and \qty{32}{GB} of memory. 
\revision{
\Cref{tab:computational_costs} provides representative computational costs for training and prediction. 
Here, \emph{Derivative evaluation} refers to the average time taken by each model to compute the state derivative~$\dot\bsx$ for a given state~$\bsx$ and input~$\bsu$, which needs to be evaluated for time integration. 
In contrast, \emph{Model integration} denotes the time required to generate a single trajectory prediction via integration. 
While the latter better reflects practical use cases, the reported integration time also depends on the integration scheme and chosen step size. 
Since we use an adaptive step size controller, the step size in turn depends on the learned dynamics.
}

\begin{table}
    \caption{\revision{Representative computational costs for training and prediction. The top two rows report the training times for derivative and trajectory fitting, respectively. The bottom two rows show the time for a single derivative evaluation and a full trajectory integration using data and models from \cref{sec:chicken_data}.}}
    \label{tab:computational_costs}
    \begin{center}
    \begin{small}
        \begin{tabular}{lrrrrrr}
            \toprule
            & \mymodel & \mymodel-LM & \gls{sNODE} & \gls{bPHNN} & \PHNN & \gls{NODE} \\
            \midrule
            Derivative fit \cref{sec:spinning_rigid_body} (\unit{\minute})  & 0.577  & 0.644  & 0.672  & 0.635  & 0.539  & 0.346 \\
            Trajectory fit \cref{sec:chicken_data} (\unit{\minute})         & 37.663 & 39.160 & -      & 38.329 & 19.111 & 17.951 \\
            \midrule
            Derivative evaluation (\unit{\milli\second})       & 0.007 & 0.007 & - & 0.007 & 0.007 & 0.005 \\
            Model integration (\unit{\milli\second})           & 0.365 & 0.393 & - & 0.422 & 0.237 & 0.221 \\
        \bottomrule
        \end{tabular}
    \end{small}
    \end{center}
\end{table}

\section{Spinning rigid body system} \label{sec:spinning_rigid_body_apendix}
\revision{The governing equations of the three-dimensional spinning rigid body used in the experiment in \cref{sec:spinning_rigid_body} are described in \cref{sec:phs}.}
We generated $(\bsomega, \dot\bsomega)$ pairs for training by evaluating \cref{eq:euler_equations} for varying $\bsomega$. To simulate a measurement process, we selected $\bsomega$ values from \qty{50}{\second} long numerically integrated trajectories, using \num{10} randomly chosen initial conditions $\bsomega(0) \in [0,1]^3$ with parameters $\bsI = \diag(1, 2, 3)$ and $\mu = 0.01$.
A representative trajectory of $\bsomega$ is shown in \cref{fig:spinning_rigid_body_training_data}.

\begin{figure}[ht]
    \centering
    \inputTikzWithExternalization{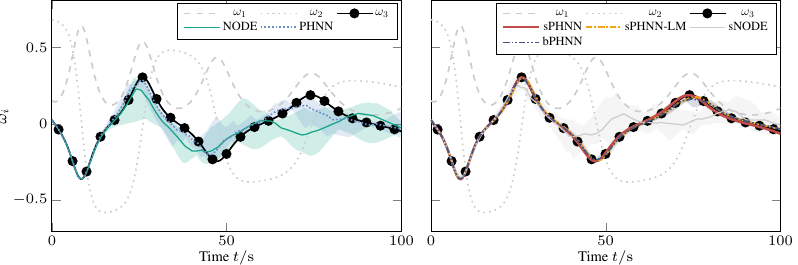}{tikz/spinning_rigid_body/training_trajectory.tex}
    \caption{{Representative, extended training trajectory for spinning rigid body, including interquartile mean and range of model predictions for $\omega_3$. Left: Models without stability bias; Right: Models with stability bias.}}
    \label{fig:spinning_rigid_body_training_data}
    \end{figure}

\section{Cascaded tanks data} \label{sec:cascaded_tanks_appendix}

The cascaded tanks dataset \cite{schoukens2020, schoukens2016} used in \cref{sec:cascaded_tanks} consists of measurements of a fluid level control system. 
Liquid flows from an upper tank through an outlet into a lower tank, exits via a similar outlet into a reservoir, and is pumped back to the upper tank, see \cref{fig:cascaded_tanks_setup}. The system's input $u$ is the voltage applied to the pump, with higher values resulting in a faster volume flow.
When the pump speed is large enough, an overflow can occur in the upper tank.
A capacitive sensor measures the fluid level in the lower tank, and the resulting voltage represents the system’s output $y$. 
The dataset consists of two trajectories comprising 1024 pairs $(u, y)$ each and \cref{fig:cascaded_tanks.training_data} depicts the training trajectory. 
The initial water level in the upper tank is unknown, but it is the same for both trajectories.

Besides the experiments presented in \cref{sec:cascaded_tanks}, we further investigated the models' capabilities to learn the hard saturation that the cascaded tanks data system exhibits at an output value of \qty{10}{\volt}. This occurs when the tanks overflow in response to a large input voltage over an extended period of time (see \cref{fig:cascaded_tanks.training_data}) \cite{schoukens2016}. We explore the models' behavior in this regime by constructing a pump voltage signal that linearly increases from \qty{1}{\volt} initially to \qty{8}{\volt} at \qty{2800}{\second}. 
The resulting predictions are shown in \cref{fig:cascaded_tanks_saturation}. For pump speeds up to $\sim$\qty{7}{\volt} ($\sim$\qty{2400}{\second}), the \mymodel, \mymodel-LM, and \gls{bPHNN} instances capture the saturation effect, while the \NODE and \PHNN instances fail to do so.

\begin{figure}[t]
    \centering
    \hspace*{0.5cm}
    \begin{subfigure}[b]{0.26\textwidth}
        \centering
        \resizebox{\columnwidth}{!}{%
            \inputTikzWithExternalization{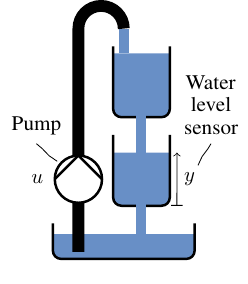}{tikz/cascaded_tanks/setup_sketch.tex}
        }
        \caption{Experimental setup}
        \label{fig:cascaded_tanks_setup}
    \end{subfigure}
        \begin{subfigure}[b]{0.6\columnwidth}
        \centering
        \inputTikzWithExternalization{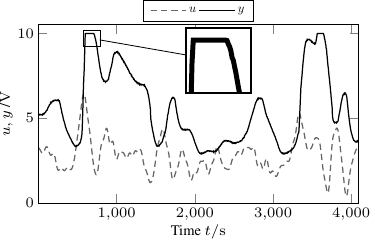}{tikz/cascaded_tanks/training_data.tex}
        \caption{Cascaded tanks training data (overflow highlighted).}
        \label{fig:cascaded_tanks.training_data}
    \end{subfigure}\\
    \vspace*{2mm}
    \begin{subfigure}[t]{\columnwidth}
        \centering\inputTikzWithExternalization{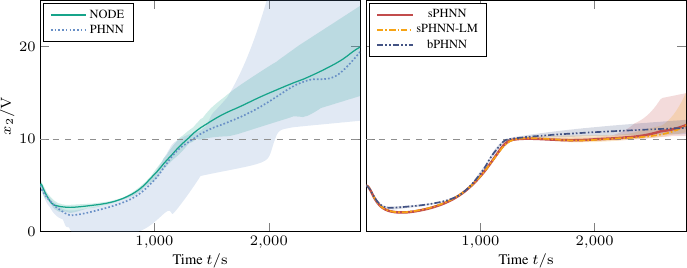}{tikz/cascaded_tanks/saturation_all_in_two.tex}
        \subcaption{Saturation.}
        \label{fig:cascaded_tanks_saturation}
    \end{subfigure}
    \caption{\emph{Cascaded tanks}: \subref{fig:cascaded_tanks_setup} System setup, \subref{fig:cascaded_tanks.training_data} Training trajectory. \subref{fig:cascaded_tanks_saturation} Predictions for a linearly increasing pump voltage. Due to an overflow, the true system exhibits a hard saturation at $x_2=\qty{10}{\volt}$. Lines correspond to the interquartile mean, whereas shaded areas represent the interquartile range of the predictions from the \num{20} model instances.}
    \label{fig:cascaded_tanks_appendix}
    \vspace*{-0.5cm}
\end{figure}

\section{Thermal food processing surrogate data} \label{sec:chicken_data_appendix}
This section provides additional details about the data used in \cref{sec:chicken_data}, which is a subset of the data generated and used by \citet{kannapinn2022} in the creation of \glspl{ROM}. The dataset describes the thermal processing of chicken breasts in a convection oven and was generated via the simulation of a soft-matter model for meats implemented in COMSOL Multiphysics.
Each trajectory in the dataset describes the temperature history at two points inside the \gls{FOM}. These result from a predefined excitation signal that controls the oven temperature. 
The probe points are situated in the center and close to the surface of the food item, with the corresponding temperatures labeled as $T_A$ and $T_B$, respectively. 
To facilitate cross-referencing, the alphanumeric identifiers used by \citet{kannapinn2022} for the individual trajectories are retained. The training data consists of trajectories \num{745} and \num{795}. 
They use an \gls{APRBS} and a multi-sine as the excitation signal $T_{\text{oven}}$. 
\Cref{fig:chicken_data_training_data} (left and middle) illustrates $T_{\text{oven}}$ along with the core and surface temperatures for both trajectories. 
This selection provides a fair evaluation of the models across the relevant temperature spectrum.
\revision{For studying the impact of the number of training trajectories on the models, we additionally included \num{13} trajectories with identifiers \num{835}, \num{851}, \num{853}, \num{861}, \num{868}, \num{873}, \num{874}, \num{888}, \num{899}, \num{902}, \num{905}, \num{906} and \num{917}.}
In the experiment, the test group AP15 \cite{kannapinn2022} is utilized as test data. It comprises \num{15} trajectories with \gls{APRBS} excitations, specifically selected to ensure an even distribution of amplitudes and median values. 
\Cref{fig:chicken_data_training_data} (right) shows an exemplary test trajectory, including the predictions for $T_A$ from all models using \num{3} augmented dimensions.
All trajectories span \qty{1395}{\second} and consist of \num{280} samples. 
The data is normalized for training to ensure the numerical values are of magnitude one. For this, an affine transformation is applied to $T_{\text{oven}}$, $T_A$, and $T_B$ individually. It shifts the ambient temperature to a value of zero and rescales the signal to have unit variance. The scaling and shifting factors are calculated using only the training data to ensure that the test data does not influence the training procedure.

\begin{figure}[ht]
    \centering
        \inputTikzWithExternalization{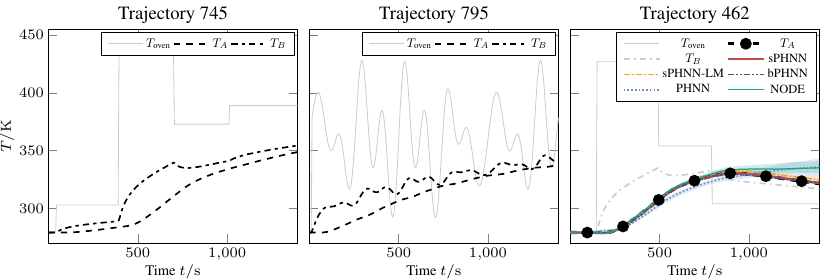}{tikz/chicken_data/dataset_examples.tex}
    \caption{\emph{Left and middle}: Training trajectories. \emph{Right}: Exemplary test trajectory with interquartile mean and range of the model predictions for $T_A$.}
    \label{fig:chicken_data_training_data}
    \end{figure}

\section{Thermal field data} \label{sec:thermal_field_data_appendix}

\subsection{Dimensionality reduction}
The following describes the \gls{POD} method for dimensionality reduction of the field data applied in \cref{sec:thermal_field_data}. 
The process is first described for general field data given as a matrix $\bsA = [\bsa_0, \dots, \bsa_M]^\intercal$ of $M$ snapshots $\bsa_i\in\bbR^N$, where $M$ is the number of snapshots and $N$ is the number of \gls{FE} nodes.
To identify the dominant modes of the data, we perform a \gls{SVD} of $\bsA$:
\begin{equation}
    \bsA = \bsU\bsSigma\bsV^{\intercal}\in\bbR^{M\times N},
\end{equation}
where $\bsU\in\bbR^{M\times M}$ and $\bsV\in\bbR^{N\times N}$ are unitary matrices and $\bsSigma\in\bbR^{M\times N}$ is a rectangular diagonal matrix containing the singular values $\sigma_1\geq \dots \geq \sigma_M$ on its diagonal. 
Dimensionality reduction is achieved by truncating the \gls{SVD} $\bsA\approx\tilde\bsU\tilde\bsSigma\tilde\bsV^{\intercal}$, where $\tilde\bsU\in\bbR^{M\times n}$ and $\tilde\bsV\in\bbR^{N\times n}$ contain the first $n$ columns of $\bsU$ and $\bsV$, and $\tilde\bsSigma=\diag(\sigma_1,\dots,\sigma_n)$ is the truncated matrix of the dominant singular values. 
The dominant modes in the data are then given by $\bsM=c\tilde\bsV$. Here, the scaling factor $c$ is the standard deviation of the elements in $\tilde\bsU\tilde\bsSigma$. This scaling ensures that the latent trajectories will have unit variance. 
To map a latent state $\bsx\in\bbR^{n}$ to the full-order field $\bsX\in\bbR^{N}$ we compute:
\begin{equation}\label{eq:svd_latent_to_fom}
    \bsX = \bsM\bsx.
\end{equation}
Conversely, to map from the full-order field to the latent space, we solve the least-squares regression problem:
\begin{equation}\label{eq:svd_fom_to_latent}
    \bsx=\argmin_{\bsx'}\norm{\bsX - \bsM\bsx'} = \frac{1}{c^2}\bsM^\intercal\bsX.
\end{equation}
In \cref{sec:thermal_field_data}, we consider two fields: The thermal field describing the state of the system and the excitation given by the heat source field. For each, a separate \gls{SVD} is performed to compute $\bsM_T$ and $\bsM_q$ using the respective snapshots from the two training trajectories. 
We thus obtain two pairs of encoders $\psi$ and decoders $\phi$:
\begin{align}
    \psi_T(\bsT) &= \argmin_{\bsx}\norm{\bsT - \bsM_T\bsx}, & \phi_T(\bsx) &= \bsM_T\bsx,\\
    \psi_q(\bsq) &= \argmin_{\bsu}\norm{\bsq - \bsM_q\bsu}, & \phi_q(\bsu) &= \bsM_q\bsu.
\end{align}

To fully leverage the benefits of \mymodels, it is essential to ensure that the absence of a heat source is encoded as the zero-vector in the corresponding latent space. Since the mapping from the latent space \cref{eq:svd_latent_to_fom} is linear, and $\bsq=\bszero$ corresponds to the absence of a heat source in the full-order space, this requirement is automatically fulfilled.
However, for applications where a nonzero state $\tilde\bsx$ describes the absence of an external energy source or for general nonlinear latent mappings (such as \glspl{AE}), the condition can be met by applying a constant translation in the latent space, which shifts the encoded state $\tilde\bsx = \psi(\tilde\bsx)$ to the origin. 
In \cref{sec:thermal_field_data}, we use this method to ensure that the thermal equilibrium corresponds to the origin of the temperature latent space. While this is not strictly necessary, knowledge of the equilibrium location in the latent space enables its integration into the \mymodel's architecture by fixing the stable equilibrium position. 
Consequently, we arrive at the representation of the latent state
\begin{align}
    \bsx &= \psi_T(\bsT) - \psi_T(\bsT_{\text{eq}}), & \bsT &\approx \phi_T(\bsx + \psi_T(\bsT_{\text{eq}})).
\end{align}
where $\bsT_{\text{eq}}$ is the full-order state vector describing thermal equilibrium.
The complete model architecture is represented in \cref{fig:ded_architecture}. There, the function $\bsf$ is used as a placeholder for an arbitrary evolution equation, to generalize to all model architectures used in \cref{sec:thermal_field_data}.

\begin{figure}[ht]
    \centering
    \resizebox{0.8\textwidth}{!}{
    \inputTikzWithExternalization{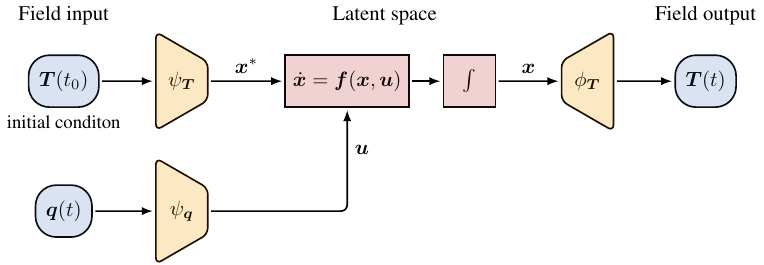}{tikz/thermal_ded/full_architecture.tex}}
        \caption{Model setup used in \cref{sec:thermal_field_data}. The initial temperature field $\bsT(t_0)$ and input field $\bsq(t)$ are mapped via encoders $\phi_T, \phi_q$ into the latent space, where the latent dynamics $\bsf$ are integrated. Here, $\bsf$ is used as a placeholder for a specific model, e.g., \NODE or \mymodel. Finally, the predicted latent trajectory $\bsx(t)$ is mapped to the full-order field via the decoder $\phi_T$.}
    \label{fig:ded_architecture}
    \vskip -0.2cm
\end{figure}

Finally, we note that alternative dimensionality reduction techniques could be applied. In preliminary experiments, we explored the use of autoencoders for this purpose. Although autoencoders offer advantages over the linear \gls{SVD} regarding achievable compression rates, they exhibited poor extrapolation performance when trained on sparse data. This issue could potentially be mitigated by training the autoencoders with additional synthetic data, which would not necessitate costly simulations; however, further research is needed to explore this approach.

\subsection{Additional results}\label{sec:thermal_field_data_appendix_additional_results}
We briefly present additional results from \cref{sec:thermal_field_data} to demonstrate that the differences in predictions across the various model architectures remain consistent across multiple instances.
\Cref{fig:ded_mesh_rmses_over_time} depicts the \glspl{RMSE} of each model from \cref{sec:thermal_field_data} computed over all \num{25} trajectories and the spatial domain. The behavior of the models with stability bias can be distinguished from the \NODE and \PHNN, which suffer from instabilities that result in rapidly increasing errors.
While both \mymodel variants and \glspl{bPHNN} exhibit decreasing errors as the system returns to thermal equilibrium, \mymodels achieve a peak \gls{RMSE} nearly an order of magnitude lower, due to their strong inductive bias from enforcing global stability.
\revision{The \mymodel-LM performs almost identically to the \mymodel initially, while showing slightly larger errors in the convergence to equilibrium. This is due to slight errors in the learned equilibrium position, which can be addressed by increasing the amount of training data, as demonstrated in \cref{sec:chicken_data}.}

\begin{figure}[ht]
    \centering
    \inputTikzWithExternalization{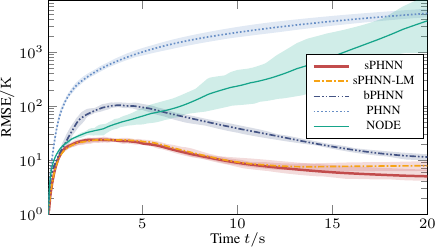}{tikz/thermal_ded/errors_all_trajectories.tex}
        \caption{Interquartile mean and range of the \glsxtrshortpl{RMSE} evaluated with all \num{25} trajectories of the thermal field data.}
    \label{fig:ded_mesh_rmses_over_time}
\end{figure}

\end{document}